\documentclass[english,a4paper,11pt]{article}
\PassOptionsToPackage{pdfpagelabels=false}{hyperref}

\usepackage{amsmath, amsxtra, amsfonts, amssymb, amstext}
\usepackage{amsthm}
\usepackage{booktabs}
\usepackage{fullpage}
\usepackage{nicefrac}
\usepackage{xspace}
\usepackage[noadjust]{cite}
\usepackage{url}
\usepackage{graphics}
\usepackage[usenames,dvipsnames]{xcolor}
\usepackage[colorlinks]{hyperref}
\definecolor{linkblue}{rgb}{0.1,0.1,0.8}
\hypersetup{colorlinks=true,linkcolor=linkblue,filecolor=linkblue,urlcolor=linkblue,citecolor=linkblue}
\usepackage[algo2e,ruled,vlined,linesnumbered]{algorithm2e}
\usepackage{wrapfig}
\newtheorem{theorem}{Theorem}
\newtheorem{lemma}[theorem]{Lemma}

\newcommand{\N}{\mathbb{N}}
\newcommand{\R}{\mathbb{R}}
\newcommand{\Z}{\mathbb{Z}}

\renewcommand{\epsilon}{\varepsilon}
\newcommand{\eps}{\epsilon}

\DeclareMathOperator{\eq}{eq}
\DeclareMathOperator{\mut}{mut}
\DeclareMathOperator{\cross}{cross}

\newcommand{\Bin}{\mathcal{B}}

\newcommand{\assign}{\leftarrow}

\newcommand{\oea}{$(1 + 1)$~EA\xspace}

\newcommand{\ga}{$(1 + (\lambda,\lambda))$~GA\xspace}

\newcommand{\onemax}{\textsc{OneMax}\xspace}
\newcommand{\OM}{\textsc{Om}\xspace}

\begin{document}
        
\title{A Tight Runtime Analysis of the $(1+(\lambda, \lambda))$ Genetic Algorithm on OneMax}

\author{
Benjamin Doerr, \'Ecole Polytechnique, Univ. Paris-Saclay, France\\
Carola Doerr, CNRS and Sorbonne Universit\'es, UPMC Univ Paris 06, Paris, France
}
\date{This is a preliminary version of a paper that is to appear at Genetic and Evolutionary Computation Conference (GECCO 2015).}
\maketitle
 
\begin{abstract}
Understanding how crossover works is still one of the big challenges in evolutionary computation research, and making our understanding precise and proven by mathematical means might be an even bigger one. As one of few examples where crossover provably is useful, the $(1+(\lambda, \lambda))$ Genetic Algorithm (GA) was proposed recently in [Doerr, Doerr, Ebel. Lessons From the Black-Box: Fast Crossover-Based Genetic Algorithms. TCS 2015]. Using the fitness level method, the expected optimization time on general \textsc{OneMax} functions was analyzed and a $O(\max\{n \log(n) / \lambda, \lambda n\})$ bound was proven for any offspring population size $\lambda \in [1..n]$.

We improve this work in several ways, leading to sharper bounds and a better understanding of how the use of crossover speeds up the runtime in this algorithm. We first improve the upper bound on the runtime to $O(\max\{n \log(n) / \lambda, n \lambda \log\log(\lambda)/\log(\lambda)\})$. This improvement is made possible from observing that in the parallel generation of $\lambda$ offspring via crossover (but not mutation), the best of these often is better than the expected value, and hence several fitness levels can be gained in one iteration.

We then present the first lower bound for this problem. It matches our upper bound for all values of $\lambda$. This allows to determine the asymptotically optimal value for the population size. It is  $\lambda = \Theta(\sqrt{\log(n) \log\log(n)/\log\log\log(n)})$, which gives an optimization time of $\Theta(n \sqrt{\log(n) \log\log\log(n) / \log\log(n)})$. Hence the improved runtime analysis both gives a runtime guarantee improved by a super-constant factor and yields a better actual runtime (faster by more than a constant factor) by suggesting a better value for the parameter $\lambda$.

We finally give a tail bound for the upper tail of the runtime distribution, which shows that the actual runtime exceeds our runtime guarantee by a factor of $(1+\delta)$ with probability $O((n/\lambda^2)^{-\delta})$ only.
\end{abstract}

\sloppy{
\section{Introduction}

The role of crossover in evolutionary computation is still not very well understood. On the one hand, it is used intensively in practice, on the other hand, few rigorous theoretical or experimental investigations support clearly the usefulness of crossover. The difficulties both in experimentally supporting early explanation models like the building block hypothesis~\cite{MitchellHF93} or in  constructing artificial example functions (e.g.~\cite{JansenW02}) such that simple evolutionary algorithms perform better with crossover than without (note that this would still only be a very weak support for the concept of crossover) rather suggest that crossover is not that easily employed with success. In the meantime, a few less artificial examples were found where crossover provably leads to a runtime having a smaller asymptotic order of magnitude, namely a simplified variant of the Ising model on rings~\cite{FischerW04}) and on trees~\cite{Sudholt05} as well as a series of works on the all-pairs shortest path problem (\cite{DoerrHK12}, \cite{DoerrT09}, and \cite{DoerrJKNT13tcs}). 

In this work, we build on the latest algorithm where crossover was proven to be useful, the \ga proposed in \cite{DoerrDE15}. Unlike most previous works, here crossover has a super-constant speed-up even for very simple functions like the \onemax test function. However, experiments show that the algorithm performs well also on linear functions, royal road functions, and maximum satisfiability instances~\cite{GoldmanP14} (in fact, the latter work shows that the \ga outperforms hill climbers for several problems, for MaxSat it also outperforms the Linkage Tree Genetic Algorithm). The \ga uses crossover in a different way than the previous works. Instead of trying to combine particularly fit solution parts, here a biased uniform crossover is used as a repair mechanism. Having such a crossover-based repair mechanism, we can use mutation with a higher rate, leading to a faster exploration of the search space. 

Given this different working-principle not seen before in discrete evolutionary optimization, there is a strong motivation to gain a deeper understanding of the \ga and its working principles. While in the first analysis of the \ga on the \onemax test function $\OM:\{0,1\}^n \rightarrow \{0,1,\ldots,n\}; x \mapsto \sum_{i=1}^n{x_i}$ only an upper bound of $O(\max\{n \log(n)/\lambda, \lambda n\})$ for the expected number of fitness evaluations was shown, we now determine the precise expected runtime to be $\Theta(\max\{n \log(n) / \lambda, n \lambda \log\log(\lambda)/\log(\lambda)\})$ for all values of $\lambda \in [1..n]$. We thus both improve the upper bound and give the first lower bound, which in addition matches the upper bound. 

We further prove a strong concentration result for the runtime, showing that deviations above the expectation are unlikely: For all $\delta > 0$, the probability that the actual runtime exceeds a bound as above by a factor of $(1+\delta)$ is at most $O((n/\lambda^2)^{-\delta})$. 

Our result on the expected runtime has two implications beyond showing that the \ga is slightly (but still by a super-constant factor) faster than what the previous work guaranteed. (i) The previous upper bound gave the best runtime guarantee for $\lambda = \Theta(\sqrt{\log n})$, namely $O(n \sqrt{\log n})$ with the old bound (and $\Theta(n \sqrt{\log n})$ with our new bound). From our new and sharp runtime estimate, however, we derive a better value for $\lambda$ (together with the guarantee that it is asymptotically optimal), namely $\lambda = \Theta(\sqrt{\log(n) \log\log(n)/\log\log\log(n)})$. This gives an optimization time of 
$\Theta(n \sqrt{\log(n) \log\log\log(n) / \log\log(n)})$. Hence we also improve the performance by determining a better value for the parameter $\lambda$. (ii) Our analysis leading to these results also gives the desired additional insights in the working principles of this crossover operator and its interplay with mutation. The improved runtime guarantee is based on the observation that when generating $\lambda$ offspring in parallel, some have a fitness significantly better than expected. We exploit this to show that sufficiently often we gain sufficiently many fitness levels in one generation. Interestingly, the good runtimes shown for the \ga only stem from better-than-expected individuals in the crossover phase, but not in the mutation phase.

With the few runtime results on crossover-based EAs and, also, still the majority of the runtime results for EA in general being for evolutionary algorithms having trivial population sizes, we feel that our work also advances the state of the art in terms of analysis methods. Our argument that one out of $\lambda$ offspring can have a significantly better fitness than the expected fitness of one offspring resembles a similar one made by Jansen, De Jong, and Wegener~\cite{JansenJW05}, who used multiple fitness level gains to prove that for the $(1+\lambda)$ EA optimizing \onemax a linear speed-up (compared to $\lambda=1$) exists if and only if $\lambda = O(\log(n) \log\log(n) / \log\log\log(n))$. Note that this result is different from ours in two respects, namely in that it does not show a positive influence of $\lambda$ on the optimization time (number of fitness evaluation), but only on the number of generations, and in that there the better-than-expected offspring is generated by mutation, whereas by crossover in our setting. A second difference is that the random experiment producing the new generation for the \ga has stochastic dependencies that are not present in the $(1+\lambda)$ EA, which require different arguments (e.g., a balls-into-bins argument). 

\section{The \texorpdfstring{\ga}{(1+(lambda,lambda))~GA}}
\label{sec:algorithm}

The \ga is a simple evolutionary algorithm using crossover. Following~\cite{DoerrDE13,DoerrDE15} we present it here for the maximization of pseudo-Boolean functions $f: \{0,1\}^n \rightarrow \R$. Its pseudo-code is given in Algorithm~\ref{alg:GA}. 

\begin{algorithm2e}[t]%
	\textbf{Initialization:} 
	Choose $x \in \{0,1\}^n$ uniformly at random and evaluate $f(x)$\;
 \textbf{Optimization:}
\For{$t=1,2,3,\ldots$}{
%%%%
\textbf{Mutation phase:}
Sample $\ell$ from $\Bin(n,p)$\label{line:L}\;
\For{$i=1, \ldots, \lambda$\label{line:mutstart}}{
$x^{(i)} \assign \mut_{\ell}(x)$ and evaluate $f(x^{(i)})$\;%,; //mutation step
}
Choose $x' \in \{x^{(1)}, \ldots, x^{(\lambda)}\}$ with $f(x')=\max\{f(x^{(1)}), \ldots, f(x^{(\lambda)})\}$ u.a.r.\label{line:mutend}\;
%%%%
\textbf{Crossover phase:}
\For{$i=1, \ldots, \lambda$\label{line:costart}}{
$y^{(i)} \assign \cross_{c}(x,x')$ and evaluate $f(y^{(i)})$\label{line:co}\; %,;//crossover step
}
Choose $y \in \{y^{(1)}, \ldots, y^{(\lambda)}\}$ with $f(y)=\max\{f(y^{(1)}), \ldots, f(y^{(\lambda)})\}$ u.a.r.\label{line:coend}\;
\textbf{Selection step:}
\lIf{$f(y)\geq f(x)$}{$x \assign y$\;%,;//selection step
}%if{$f(x')>f(x)-\lambda$} 
}%t=1,...
\caption{The \ga from~\cite{DoerrDE13} with offspring population size $\lambda$, mutation probability $p$, and crossover bias $c$. The standard choices for the latter two parameters are $p = \lambda / n$ and $c = 1 / \lambda$. The mutation operator $\mut_\ell$ generates an offspring from one parent by flipping exactly $\ell$ random bits (without replacement). The crossover operator $\cross_c$ performs a biased uniform crossover, taking bits independently with probability $c$ from the second argument and with probability $1-c$ from the first parent.}
\label{alg:GA}
\end{algorithm2e}

The \ga is initialized with a solution candidate drawn uniformly at random from $\{0,1\}^n$. It then proceeds in iterations (rounds) consisting of a mutation, a crossover, and a selection phase. In an important contrast to many other genetic algorithms, the mutation phase \emph{precedes} the crossover phase. This allows to use crossover as a repair mechanism, as we shall discuss in more detail below.

In the \emph{mutation phase} of the \ga, we create $\lambda$ offspring from the current-best solution $x$ by applying to it the mutation operator $\mut_{\ell}(\cdot)$, which samples $\ell$ (different) positions uniformly at random and generates a new bit string from the input by flipping the bits in these $\ell$ positions. That is, $\mut_\ell(x)$ is a bit string in which for $\ell$ random positions $i$ the entry $x_{i} \in \{0,1\}$ is replaced by $1-x_i$. 
The \emph{step size} $\ell$ is chosen randomly according to a binomial distribution $\Bin(n,p)$ with $n$ trials and success probability $p$; in our analyses we follow the suggestion in~\cite{DoerrDE15} and use $p=\lambda/n$. The expected distance of a random offspring to $x$ is thus $\lambda$. To ensure that all mutants have the same distance from the father $x$, the same $\ell$ is used for all $\lambda$ offspring. 
The fitness of the $\lambda$ offspring is computed and the best one of them,  $x'$, is selected to take part in the crossover phase. If there are several offspring having maximal fitness, we pick one of them uniformly at random (u.a.r.).

When $x$ is already close to an optimal solution, the offspring created in the mutation phase are typically all of much worse fitness than $x$. Our hope is though that they have discovered some parts of the optimum solution that is not yet reflected in $x$. In order to preserve these parts while at the same time not destroying the good parts of $x$, the \ga creates in the \emph{crossover phase} $\lambda$ offspring from $x$ and $x'$. Each one of these offspring is sampled from a uniform crossover with bias $c$ to take an entry from $x'$. That is, each of the offspring $\cross_c(x,x')$ is created by taking independently for each position $i$ the entry $x'_i$ with probability $c$ and taking the entry from $x_i$ otherwise. 
Again we follow the suggestion from~\cite{DoerrDE15} and use $c=1/\lambda$. With this choice, we try to ensure that the new individual $y$ is close to $x$ even then the mutation rate is high. Note that $\cross_c(x,\mut_\lambda(x))$ is just the outcome of applying to $x$ standard-bit mutation with mutation probability $1/n$. Of course, due to the intermediate selection of $x'$, our new individual $y$ follows a more complicated distribution (which is the reason for the performance of our algorithm). Again we evaluate the fitness of the $\lambda$ new offspring and select the best one of them, which we denote by $y$. If there are several offspring of maximal fitness, we simply take one of them uniformly at random.\footnote{In~\cite[Section 4.4]{DoerrDE13} and~\cite{DoerrDE15} a different selection rule is suggested for the crossover phase, which is more suitable for functions with large plateaus of the same fitness value. 
Since we consider in this work only the \onemax function, for which both algorithms are identical by symmetry reasons, we refrain from stating in Algorithm~\ref{alg:GA} the slightly more complicated version proposed there, which selects the parent solution $x$ only if there is no offspring $\neq x$ of fitness value at least as good as the one of $x$.}

Finally, in the \emph{selection step} the current-best solution $x$ is replaced by its offspring $y$ if and only if the fitness of $y$ is at least as good as the one of $x$.

As common in the runtime analysis community, we do not specify a termination criterion. The simple reason is that we study as a theoretical performance measure the expected number of function evaluations that the \ga performs until it evaluates for the first time a search point of maximal fitness (the so-called optimization time). Of course, for an application to a real problem a termination criterion has to be specified.

\section{Runtime Analysis}

\emph{Runtime analysis} is one of the most successful theoretical tools to understand the performance of evolutionary algorithms. The \emph{runtime} or \emph{optimization time} of an algorithm (e.g., our \ga) on a problem instance (e.g., the \onemax function) is the number of fitness evaluations that are performed until for the first time an optimal solution is evaluated. 

If the algorithm is randomized (like our \ga), this is a random variable $T$, and we usually make statements on the expected value $E[T]$ or give bounds that hold with some high probability, e.g., $1 - 1/n$. When regarding a problem with more than one instance (e.g., traveling salesman instance on $n$ cities), we take a worst-case view. This is, we regard the maximum expected runtime over all instances, or we make statements like that the runtime satisfies a certain bound for all instances. 

In this work, the optimization problem we regard is the classic \onemax test problem consisting of the single instance $\OM: \{0,1\}^n \to \{0,1,\ldots,n\}; x \mapsto \sum_{i = 1}^n x_i$, that is, maximizing the number of ones in a bit-string. 
Despite the simplicity of the \onemax problem, analyzing randomized search heuristics on this function has spurred much of the progress in the theory of evolutionary computation in the last 20 years, as is documented, e.g., in the recent textbook~\cite{Jansen13}. 

Of course, when regarding the performance on a single test instance, then we should ensure that the algorithm does not exploit the fact that there is only one instance. A counter-example would be the algorithm that simply evaluates and outputs $x^* = (1,...,1)$, giving a perfect runtime of $1$. One way of ensuring this is that we restrict ourselves to unbiased algorithms (see~\cite{LehreW12}) which treat bit-positions and bit-values in a symmetric fashion. Consequently, an unbiased algorithm for the \onemax problem has the same performance on all problems with isomorphic fitness landscape, in particular, on all (generalized) \onemax functions $\OM_z : \{0,1\}^n \to \{0,1,\ldots,n\}; x \mapsto \eq(x,z)$ for $z \in \{0,1\}^n$, where $\eq(x,z)$ denotes the number of bit-positions in which $x$ and $z$ agree. It is easy to see that the \ga is unbiased (for all parameter settings). We will henceforth not stress this fact anymore. Any other algorithms we will discuss are also all unbiased.

The result we build on is this work is the runtime analysis of the \ga for various parameter settings on the \onemax problem in~\cite{DoerrDE15}. This analysis suggested, in particular, certain settings for the parameters to which we will restrict ourselves in the following. For these settings, the following upper bound for the expected runtime was proven (Theorem 4 in \cite{DoerrDE15} for $\lambda=k$, note that the case $\lambda=1$ excluded there is trivial for $\lambda = k$ since in this case, the \ga imitates the \oea). 

\begin{theorem}[\cite{DoerrDE15}]
\label{thm:old}
Let $\lambda \in [1..n]$, possibly depending on $n$. The expected optimization time of the \ga with mutation probability $p=\lambda/n$ and crossover bias $c=1/\lambda$ on \onemax is \[O\bigg(\max\bigg\{\frac{n \log n}{\lambda}, \lambda n \bigg\}\bigg).\]

In particular, for $\lambda = \Theta(\sqrt{\log n})$, the expected optimization time is of order at most $n \sqrt{\log n}$.
\end{theorem}

The main result of this work is the following improvement and strengthening of the previous result.

\begin{theorem}[our main result]
\label{thm:ourmain}
  Let $\lambda \in [1..n]$.   The expected optimization time of the \ga with mutation probability
  $p = \lambda / n$ and 
  crossover bias $c = 1 / \lambda$ on the \onemax test function is 
  \begin{align*}
  \Theta\left(\max\left\{\frac{n \log(n)}{\lambda}, \frac{n \lambda \log\log(\lambda)}{\log(\lambda)}\right\}\right).
  \end{align*}
For all $\delta > 0$, the probability that the actual runtime exceeds a bound of this magnitude by a factor of more than $(1+\delta)$ is at most $O((n/\lambda^2)^{-\delta})$.
 
The expected runtime is minimized by the parameter choice 
$\lambda = \Theta(\sqrt{\log(n) \log\log(n) / \log\log\log(n)})$. This yields an expected optimization time of 
\begin{align*}
\Theta(n \sqrt{\log(n) \log\log\log(n) / \log\log(n)}).
\end{align*}
\end{theorem}

\section{Notation and Technical Tools}

In this section, besides fixing some very elementary notation, we collect the main technical tools we shall use. Mostly, these are large deviations bounds of various types. For the convenience of the reader, we first state the known ones. We then prove a tail bound for sums of geometric random variables with expectations bounded from above by the reciprocals of the first positive integers. We finally state the well-known additive drift theorem.

\subsection{Notation}

As the reader has experienced already, we write $[a..b]$ to denote the set $\{z \in \Z \mid a \le z \le b\}$ of integers between $a$ and $b$. We also use the short-hand $[a]:= [1..a]$. We write $\log(n)$ to denote the binary logarithm of $n$ and $\ln(n)$ to denote the natural logarithm of $n$. However, to avoid unnecessary case distinctions, we define $\log(n):=1$ for all $n \le 2$ and $\ln(n) := 1$ for all $n \le e$.

\subsection{Known Chernoff Bounds}

The following large deviation bounds are well-known and can be found, e.g., in~\cite{Doerr11bookchapter}. We call all these bounds Chernoff bounds (CB) despite the fact that it is now known that some have been found earlier by other researchers.

\begin{theorem}[classic Chernoff bounds]\label{thm:chernoff} 
  Let $X_1, \ldots, X_n$ be independent random variables taking values in $[0,1]$. Let $X = \sum_{i = 1}^n X_i$. 
  \begin{enumerate}
    \item \label{multchernoffupperstrong} Let $\delta \ge 0$. Then \\$\Pr[X \ge (1+\delta) E[X]] \le (\frac{e^\delta}{(1+\delta)^{1+\delta}})^{E[X]}$. 
    \item \label{multchernoffupper} Let $\delta \in [0,1]$. Then \\$\Pr[X \ge (1+\delta) E[X]] \le \exp(-\delta^2 E[X]/3)$. 
    \item \label{multchernofflower} Let $\delta \in [0,1]$. Then \\$\Pr[X \le (1-\delta) E[X]] \le \exp(-\delta^2 E[X]/2)$.
  \end{enumerate}
\end{theorem}

%Binary random variables $X_1, \ldots, X_n$ are called \emph{negatively correlated}, if for all $I \subseteq [1..n]$ the following holds. 
%\begin{eqnarray*}
%	\Pr[\forall i \in I : X_i = 0] &\le& \prod_{i \in I} \Pr[X_i = 0],\\
%	\Pr[\forall i \in I : X_i = 1] &\le& \prod_{i \in I} \Pr[X_i = 1].
%\end{eqnarray*}
Binary random variables $X_1, \ldots, X_n$ are called \emph{negatively correlated}, if for all $I \subseteq [1..n]$ we have $\Pr[\forall i \in I : X_i = 0] \le \prod_{i \in I} \Pr[X_i = 0]$ and $\Pr[\forall i \in I : X_i = 1] \le \prod_{i \in I} \Pr[X_i = 1]$.

\begin{theorem}[CB, negative correlation]\label{tchernoffnegcor}
  Let $X_1, \ldots, X_n$ be negatively correlated binary random variables. Let $a_1, \ldots, a_n \in [0,1]$ and $X = \sum_{i = 1}^n a_i X_i$. Then $X$ satisfies the Chernoff bounds given in Theorem~\ref{thm:chernoff} \ref{multchernoffupper} and \ref{multchernofflower}.
\end{theorem}

Chernoff bounds also hold for \emph{hypergeometric} distributions. Let $A$ be any set of $n$ elements. Let $B$ be a subset of $A$ having $m$ elements. If $Y$ is a random subset of $A$ of $N$ elements (chosen uniformly at random from all $N$-element subsets of $A$, then $X :=|Y \cap B|$ has a hypergeometric distribution with parameters $(n,N,m)$. 

\begin{theorem}[CB, hypergeometric distributions]\label{thm:hypergeomchernoff}
  If $X$ has a hypergeometric distribution with parameters $(n,N,m)$, then $E[X]=Nm/n$ and $X$ satisfies the Chernoff bounds given in Theorem~\ref{thm:chernoff}.
\end{theorem}

\subsection{A Chernoff Bound for Geometric Random Variables}

To prove the concentration statement in Theorem~\ref{thm:ourmain}, we need a tail bound for the upper tail of a sum of a sequence of independent geometric random variables having expectations that are upper-bounded by a multiple of the harmonic series. While generally Chernoff bounds for geometric random variables are less understood than for bounded random variables, Witt~\cite{Witt14} proves such a bound. Witt's bound is sufficient for our purposes. For two reasons, we prove the following alternative result below. (i) Our proof is a simple reduction to the well-understood coupon collector process, and thus much simpler than Witt's. (ii) At the same time, our proof gives a stronger bound (for our setting, Witt's bound on the failure probability is roughly the fourth root of ours). Since scenarios as treated here are quite common in runtime analysis (for example, they appear whenever the fitness level method is employed in a situation where the probability of a progress is inversely proportional to the fitness distance from the optimum), we feel that presenting our result is justified here. 

We say that $X$ has a geometric distribution with success probability $p$ if for each positive integer $k$ we have $\Pr[X = k] = (1-p)^{k-1} p$. For all $n \in \N$, let $H_n := \sum_{i=1}^n (1/i)$ denote the $n$th Harmonic number.

\begin{lemma}\label{lem:harmonic}
  Let $X_1, \ldots, X_n$ be independent geometric random variables with success probabilities $p_i$. Assume that there is a number $C \le 1$ such that $p_i \ge C i / n$ for all $i \in [n]$. Let $X = \sum_{i=1}^n X_i$. Then $E[X] \le (1/C) n H_n \le (1/C) n (\ln(n)+1)$ and $\Pr[X \ge (1+\delta) (1/C) n \ln(n)] \le n^{-\delta}$ for all $\delta > 0$.
\end{lemma}

\begin{proof}
  For $i \in [n]$, let $X'_i$ be a geometric random variable with success probability exactly $C i/n =: p'_i$. Let the $X'_i$ be independent. Then $X'_i$ dominates $X_i$ for all $i \in [n]$, and consequently, $X' := \sum_{i=1}^n X'_i$ dominates $X$. Recall that a random variable $Y'$ dominates a random variable $Y$ if for all $r \in \R$, $\Pr[Y \ge r] \le \Pr[Y' \ge r]$. Note that this implies that $E[Y] \le E[Y']$ and that all upper tail bounds for $Y'$ immediately take over to $Y$. Consequently, we can conveniently argue for $X'$ instead of $X$.
  
  For the statement on the expectation, we recall that the expectation of a geometric random variable with success probability $p$ is $1/p$. Consequently, by linearity of expectation, we have $E[X'] = \sum_{i=1}^n E[X'_i] = \sum_{i=1}^n (1/p'_i) = (n/C) \sum_{i=1}^n (1/i) = (n/C) H_n$. 

  For the tail bound, consider the following coupon collector process. There are $n$ different types of coupons. In each round (independently), with probability $C$ we obtain a coupon having a type chosen uniformly at random and with probability $1-C$ we obtain nothing. We are interested in the number $T$ of rounds until we have each type of coupon at least once. For $i \in [n]$, let $T_i$ denote the number of rounds needed to get a coupon of a type not yet in our possession given that we have already $n-i$ different types. In other words, $T_i$ is the time we need to go from ``$i$ types missing'' down to ``$i-1$ types missing''. We observe that $T_i$ has the same distribution as $X_i$ and that $T = \sum_{i=1}^n T_i$. Consequently, $T$ and $X$ are equally distributed. 
  
  The advantage of this reformulation is that it allows us a different view on $X' \sim T$: The probability that after $t$ rounds of the coupon collector process we do not have a fixed type is exactly $(1 - C/n)^t$. Using a union bound, we see that the probability that after $t$ rounds some coupon is missing, is at most $n(1-C/n)^t$. For $t = (1+\delta)(1/C) n \ln(n)$, this is at most $n(1-C/n)^t \le n \exp(-Ct/n) = n \exp(-(1+\delta)\ln(n)) = n^{-\delta}$.
\end{proof}

Note that in the proof of Lemma~\ref{lem:harmonic}, once we have defined the coupon collector process (but not before), we could have also used multiplicative drift. This would, however, not have given better bound, nor a shorter proof.

\subsection{Additive Drift}

Drift analysis comprises a couple of methods to derive from information about the expected progress (e.g., in terms of the fitness distance) a result about the time needed to achieve a goal (e.g., finding an optimal solution). We shall several times use the following additive drift theorem from~\cite{HeY01} (see also Theorem~2.7 in~\cite{OlivetoY11bookchapter}).

\begin{theorem}[additive drift theorem]\label{thm:drift}
  Let $X_0, X_1, ...$ be a Markov process over a finite state space $S$. Let $g: S \to \R_{\ge 0}$. Let $T := \min\{t \ge 0 \mid X_t=0\}$. Let $\delta > 0$. 
  
  (i) If for all $t$, we have $E[X_t - X_{t+1} | X_t>0] \ge \delta$, then $E[T | X_0] \le g(X_0) / \delta$. 
  
  (ii) If for all $t$, we have $E[X_t - X_{t+1} | X_t>0] \le \delta$, then $E[T | X_0] \ge g(X_0) / \delta$. 
\end{theorem}

\section{Proof of the Upper Bound}
\label{sec:strong}
\label{sec:tight}

In this section, we prove the upper bound statement of Theorem~\ref{thm:ourmain}, that is, that the \ga with standard parameter settings optimizes every \onemax function using a number of $O\left(\max\left\{\frac{n \log(n)}{\lambda}, \frac{n \lambda \log\log(\lambda)}{\log(\lambda)}\right\}\right)$ fitness evaluations both in expectation and with probability $1 - n^{-c}$, where $c$ is an arbitrary positive constant.

The proof of the previous upper bound (Theorem~\ref{thm:old}) was based on the fitness level method (first used in~\cite{Wegener01} in the proof of Theorem~1, more explicit in~\cite{Wegener02}, see also~\cite{OlivetoY11bookchapter}). In its classic version, this method pessimistically estimates the runtime via the sum of the times needed to leave each fitness level. It thus does not profit from the fact that a typical run of the algorithm might not visit every fitness level. By a more careful analysis of the mutation phase (Lemma~\ref{lem:advmut}) and the crossover phase (Lemma~\ref{lem:advcross}), we shall show that this indeed happens. For all values of $\lambda$, we obtain that when starting an iteration with a search point $x$ having \emph{fitness distance} $d(x) := n - \OM(x)$ at least $n \log\log(\lambda) / \log(\lambda)$, then the average fitness improvement is $\Omega(\log(\lambda) / \log\log(\lambda))$. Consequently, additive drift analysis (Theorem~\ref{thm:drift}) tells us that only $O(n \log\log(\lambda) / \log(\lambda))$ iterations are needed to find a search point with fitness distance at most $n \log\log(\lambda) / \log(\lambda)$. Note that the fitness range from the typical initial fitness distance of $n/2$ to a fitness distance of $n \log\log(\lambda) / \log(\lambda)$ contains $\Omega(n)$ fitness levels. Hence the previous analysis would have given only a bound of $O(n)$ rounds. 

There is an intuitive explanation for these numbers based on the balls-into-bins paradigm. When our current search point $x$ is in distance $n/D$ from the optimum, then already $x^{(1)}$ has an expected number of at least $\lambda / D$ ``good bits'', i.e., bit positions that are zero in $x$ and one in $x^{(1)}$. 
The same is true for $x'$. Each of these good bits is copied in each of the $y^{(j)}$ generated in the crossover phase with probability $1/\lambda$. 
The total number of copies of good bits in $y^{(1)}, \ldots, y^{(\lambda)}$ thus is around $\lambda/D$ again. Since they are uniformly spread over the $y^{(j)}$, we are in a situation closely resembling the \emph{balls-into-bins} scenario, in which $\lambda/D$ balls are uniformly thrown into $\lambda$ bins. By a result of Raab and Steger~\cite{RaabS98}, we know that when $D$ is at most polylogarithmic in $\lambda$, then the most-loaded bin will contain $\Theta(\log(\lambda)/\log\log(\lambda))$ balls. For our setting, this means that we expect one of the $y^{(j)}$ to inherit $\Omega(\log(\lambda)/\log\log(\lambda))$ good bits. Unfortunately, since we do not distribute the good bits completely independently, we cannot transform this intuitive argument into a rigorous proof, but need to argue differently. 

We start by analyzing the mutation phase. Since we aim at understanding those iterations where we gain more than a constant number of fitness levels, we restrict ourselves to the case that $\lambda = \omega(1)$, which eases the calculations. 

\begin{lemma}\label{lem:advmut}
  Let $\eps > 0$. Assume that $\lambda = \omega(1)$. Let $x \in \{0,1\}^n$. Let $D$ be such that $d := d(x) = n/D$. Assume that $D = o(\lambda)$. Consider one run of the mutation phase of Algorithm~\ref{alg:GA}. As in the description of the algorithm, denote by $\ell$ the actual mutation strength and by $x'$ the winner individual. Let $B' := \{i \in [n] \mid x_i = 0 \wedge x'_i = 1\}$ the set of $1$-bits that $x'$ has gained over $x$. 
  
  Then with probability $1 - o(1)$, we have both $|\ell - \lambda| \le \eps\lambda$ and $|B'| \ge (1 - \eps) \lambda / D$.
\end{lemma}

The statement follows easily from Theorem~\ref{thm:chernoff} and~\ref{thm:hypergeomchernoff}.

\begin{proof}
Since $\lambda = \omega(1)$ and $\ell$ follows a binomial distribution with parameters $n$ and $\lambda/n$, a simple application of the Chernoff bound (Theorem~\ref{thm:chernoff} \ref{multchernoffupper} and \ref{multchernofflower}) implies that with probability $1-o(1)$ we  have $|\ell - \lambda| \le (\eps/2)\lambda$. Conditional on that, we analyze how the first offspring $x^{(1)}$ is generated. Let $B_1$ be the set of bit positions that are zero in $x$ and one in $x^{(1)}$, that is, $B_1 := \{i \in [n] \mid x_i = 0 \wedge x^{(1)}_i = 1\}$ (``good bits''). Then $E[|B_1|] = d (\ell / n) = \ell / D$. Since $D = o(\lambda)$ and $\ell = \Theta(\lambda)$, this expectation is $\omega(1)$ and a Chernoff bound for the hypergeometric distribution (Theorems~\ref{thm:hypergeomchernoff} and~\ref{thm:chernoff}~\ref{multchernofflower}) shows that we have $\Pr[|B_1| \ge (1-(\eps/2)) \ell / D] = 1 - o(1)$. Since all $x^{(j)}$, $j \in [\lambda]$, have the same Hamming distance from $x$, the fittest individual $x'$ is also the one with the largest number of good bits. Hence $|B'| \ge |B_1| \ge (1-(\eps/2))\ell / D \ge (1-\eps) \lambda / D$ with probability $1-o(1)$.
\end{proof}

We next analyze a run of the crossover phase. While in the previous lemma we only exploited that an individual generated in the mutation phase has roughly as many good bits as expected, we shall now exploit that the best of the $\lambda$ individuals generated in the crossover phase  is much better than the average one. 

\begin{lemma}\label{lem:advcross}
  Let $x, x' \in \{0,1\}^n$ such that their Hamming distance $\ell := H(x,x')$ satisfies $\ell \le 2\lambda-2$. Let $D'$ be such that $B' := \{i \in [n] \mid x_i = 0 \wedge x'_i = 1\}$ satisfies $|B'| \ge \lambda / D'$. Consider a run of the crossover phase starting with these variable values and computing an offspring $y \in \{0,1\}^n$. 
  
  Then with probability at least $1 - 1/e$, we have $\OM(y) - \OM(x) \ge \lfloor \min\{(\tfrac 12 \ln(\lambda)-1) / (\ln\ln(\lambda) + \ln(D')), \lambda/D'\}\rfloor$. 
\end{lemma}

\begin{proof}
  Let $\gamma \le \lambda / D'$ be a positive integer. Consider the outcome $y^{(j)}$ of a single crossover operation for some $j \in [\lambda]$ 
  in Algorithm~\ref{alg:GA}. Let $A_j$ be the event that $\OM(y^{(j)}) \ge \OM(x) + \gamma$. 
  This event in particular occurs when the crossover operation selects $\gamma$ ``good bits'' (those with index in $B'$) from $x'$ and none of the ``bad bits'' (those, in which $x$ and $x'$ differ, but that are not in $B'$). 
  Consequently, 
\begin{align}
\label{eq:jansen}
	\Pr[A_j] 
	&\ge \binom{|B'|}{\gamma} (1/\lambda)^\gamma (1-1/\lambda)^{\ell-\gamma}\\
	&\ge \left(|B'|/\gamma \right)^\gamma (1/\lambda)^\gamma (1-1/\lambda)^{2 (\lambda-1)}\nonumber\\
	&\ge (\lambda / (D' \gamma))^\gamma (1/\lambda)^\gamma (1/e^2)\nonumber\\
	& = \exp(-2 - \gamma \ln \gamma - \gamma \ln D').\nonumber
\end{align}
For $\gamma = \lfloor \min\{(\tfrac 12 \ln(\lambda)-1) / (\ln\ln(\lambda) + \ln(D')), \lambda/D'\}\rfloor$ we have 
$\Pr[A_j] \ge 1/\lambda$. 
Consequently, the probability that at least one of the $A_j$ holds, is at least $1 - (1-1/\lambda)^\lambda \ge 1-1/e$.
\end{proof} 

We note that the argument up to~(\ref{eq:jansen}) is very similar to the reasoning in the proof of Theorem~5 in~\cite{JansenJW05}, where it is shown that the $(1+\lambda)$ EA optimizing $\onemax$ performs super-constant improvements in the early part of the optimization process. The choice of our $\gamma$, however, is different due to the different relation of $\lambda$ and $n$. Interestingly, in the analysis of the mutation phase, such arguments do not seem to give significant additional improvement (recall that there we only used the expected gain from a fixed single offspring).

Above, we showed that in the early part of the optimization process, we regularly gain more than one fitness level in one iteration. For the remainder, we re-use the fitness level type argument of~\cite{DoerrDE15}, which is summarized in the following lemma (Lemma~7 of~\cite{DoerrDE13} in the special case that $k=\lambda$).
\begin{lemma}\label{lem:improve}
  Assume $\lambda \ge 2$.  In the notation of the \ga, the probability that one iteration produces a search point $y$ that is strictly better than the parent $x$, is at least \[p_{d(x)} := C \, \bigg(1 - \bigg(\frac{n-d(x)}{n}\bigg)^{\lambda^2 / 2} \bigg),\] where $C$ is an absolute constant.
\end{lemma}

We are now ready to prove the main result of this section.

\begin{proof}[Proof of the upper bound in Theorem~\ref{thm:ourmain}]
  We regard the different regimes of the optimization process separately, since they need very different arguments. If $\lambda = \omega(1)$, then let $d_0 := n \ln\ln(\lambda) / \ln(\lambda)$, else let $d_0 = n$ (and there is no first phase).
  
  \textbf{First phase: From the random starting point to a solution $x$ with $d(x) \le d_0 := n \ln\ln(\lambda) / \ln(\lambda)$ in $O(n \log\log(\lambda) / \log(\lambda))$ iterations.} Let $\bar D = \ln(\lambda) / \ln\ln(\lambda)$. Let $x \in \{0,1\}^n$ be any search point with $d(x) > d_0 = n/\bar D$. By Lemma~\ref{lem:advmut} and~\ref{lem:advcross}, we see that with probability $p = 1 - (1/e) - o(1)$, one iteration of the main loop of Algorithm~\ref{alg:GA} produces a solution $y$ with $\OM(y) \ge \OM(x) + \Delta$, where $\Delta$ is some number satisfying $\Delta = \Omega(\log(\lambda) / \log\log(\lambda))$. This seems to call for an application of additive drift (Theorem~\ref{thm:drift}), but in particular for the derivation of the large deviation claim, the following hand-made solution seems to be easier (despite several tail bounds for additive drift existing,  see, e.g.,~\cite{Kotzing14} and the references therein).
  
  For $t = 1, 2, \dots$ lets us define the following binary random variable $X_t$. If at the start of iteration $t$ we have $d(x) > d_0$, then $X_t = 1$ if and only if $\OM(y) \ge \OM(x) + \Delta$. If $d(x) \le d_0$, let $X_t = 1$ with probability $p$ independent from all other random decisions. For all $T > 0$, we observe that $X_T := \sum_{i=1}^T X_t \ge n/\Delta$ implies that $T_0 \le T$, that is, our \ga needed at most $T$ iterations to find a search point $x$ with $d(x) \le d_0$. We have $E[X_T] = Tp$ and $\Pr[X_T \le (1/2) E[X_T]] \le \exp(-E[X_T]/8)$. In particular, for $T = 2n / (\Delta p)$, we have $E[X_T] = 2n / \Delta$ and $\Pr[X_T \le n/\Delta] \le \exp(-E[X_T]/8) = \exp(-n/(4\Delta)) = \exp(-n^{1-o(1)})$.
  
  \textbf{Second phase: From a solution with $d$-value at most $d_0$ to one with $d$-value at most $d_1 := \lfloor n/(2\lambda^2) \rfloor$ in $O(n \log\log(\lambda) / \log(\lambda))$ iterations.} Once we have a solution of  fitness distance at most $d_0$, we use the fitness level argument analogous to the proof of Theorem~\ref{thm:old}. We reformulate the proof slightly to allow proving a large deviation bound for the optimization time. By Lemma~\ref{lem:improve}, the remaining number of iterations is dominated by a sum of geometric random variables $X_{d_0}, \dots, X_{1}$ where $\Pr[X_d = m] = (1-p_d)^{m-1} p_d$ for all $m = 1, 2, \dots$ and $p_d = C \, (1 - (\frac{n-d(x)}{n})^{\lambda^2 / 2} )$ is as  in Lemma~\ref{lem:improve}.
  
  Note that for $d \ge d_1$, $p_d = C'$ for some absolute constant $C'$. Hence the expected number $T_1$ of iterations to reduce the fitness distance to $d_1$ is at most $E[T_1] = E[X_{d_0} + \dots + X_{d_1-1}] \le (1/C') (d_0-d_1) \le (1/C') d_0 = O(n \log\log(\lambda) / \log(\lambda))$ by linearity of expectation. Since each iteration with $d(x) \ge d_1$, independent of what happened in the previous iterations, has a success chance of at least $C'$, we observe that the probability to have fewer than $(d_0 - d_1)$ successes in $2(1/C')(d_0-d_1)$ iterations is at most $\exp(-(d_0-d_1)/(4C')) = \exp(-\Theta(d_0)) = \exp(-n^{1-o(1)})$. Note that to apply the (multiplicative) Chernoff bound, here we used the ``moderate independence'' argument of Lemma~1.18 of~\cite{Doerr11bookchapter}.
  
  \textbf{Third phase: From a solution with fitness distance at most $d_1$ to an optimal solution in $O(n \log(n) / \lambda^2)$ iterations.} We continue to use the fitness level method as in the previous section of the proof, but note that for $d \le d_1$, we have $p_d = C( 1 - (1 - d/n)^{\lambda^2/2}) \ge C(1 - \exp(d \lambda^2 / (2n))) \ge C d \lambda^2 / (4n)$, where we used the estimate $e^{-x} \le 1 - x/2$ valid for all $x \in [0,1]$. We thus see that the remaining time $T_2$ to get to the optimal solution is dominated by $X = X_{d_1} + \dots + X_1$, which is a sum of independent geometric random variables with harmonic expectations. Hence by Lemma~\ref{lem:harmonic}, we have $E[T_2] \le E[X] \le \frac{4n (\ln(d_1)+1)}{C \lambda^2} = O(n \log(n) / \lambda^2)$ and $\Pr[T_2 \ge (1+\delta)\frac{4n \ln(d_1)}{C \lambda^2}] \le d_1^{-\delta}$ for any $\delta > 0$.
  
  In total, we see that the number $T = T_0+T_1+T_2$ of iterations until the optimum is found has an expectation of at most $E[T] = E[T_0]+E[T_1]+E[T_2] =  O(\max\{n \log(n) / \lambda^2, n \log\log(\lambda)/\log(\lambda)\})$ and the probability that this upper bound is exceeded by a constant factor of $(1+\delta)$ is only $O((n/\lambda^2)^{-\delta})$.
  
  Since in each iteration the fitness of $O(\lambda)$ search points is computed, we proved the claimed upper bound of $O(\max\{n \log(n) / \lambda, n \lambda \log\log(\lambda)/\log(\lambda)\})$ for the expected optimization time, and again, exceeding this expectation by a factor of $1+\delta$ has a probability of only $O((n/\lambda^2)^{-\delta})$.
\end{proof}

\section{A Matching Lower Bound}\label{sec:lower}

In this section, we prove the first lower bound for the runtime of the \ga. It matches the new upper bound proven in the previous section, so the two bound determine the asymptotic runtime of the \ga for all $\lambda \in [1..n]$. This sharp runtime result immediately gives the optimal value for the population size $\lambda$.

\begin{theorem}\label{thm:preciselower}
  For all $\lambda \leq n$ the \ga with the standard parameter setting $p = \lambda / n$ and $c = 1 / \lambda$ needs an expected number of  
%  \[\Omega(\max\{n \log(n) / \lambda, n \lambda \log\log(\lambda)/\log(\lambda)\}).\] 
  \[\Omega\left(\max\left\{\frac{n \log(n)}{\lambda}, \frac{n \lambda \log\log(\lambda)}{\log(\lambda)}\right\}\right)\]
  fitness evaluations to find the optimum of any \onemax function. 
%  Consequently, $\lambda = \Theta(\sqrt{\log(n) \log\log(n) / \log\log\log(n)})$ is the optimal choice for the parameter $\lambda$ and this yields an optimization time of  
%  $\Theta(n \sqrt{\log(n) \log\log\log(n) / \log\log(n)})$. 
\end{theorem}

To prove the theorem, we show that the expected optimization time is both $\Omega(\frac{n \log(n)}{\lambda})$ and $\Omega(\frac{n \lambda \log\log(\lambda)}{\log(\lambda)})$ (at least for sufficient ranges of $\lambda$). We do this separately in the following two subsections. The proof of Theorem~\ref{thm:preciselower} then is an immediate consequence of Lemmas~\ref{lem:lower1} and~\ref{lem:lower2}.

We remark without proof that also the lower tail of the runtime distribution admits tail bounds. Since such bounds are less relevant for the use of algorithms, we do not give further details.

\subsection{First Lower Bound}

To prove that $\Omega(n \log(n) / \lambda)$ is a lower bound, we use the standard argument that in order to find the optimum at least each bit that was not initially set to one has to be touched at least once by a mutation operator. This argument has been used, e.g., in the classic proof for the lower bound of the runtime of the \oea in~\cite{DrosteJW02}. We have to be slightly more careful though, because our random experiment has two types of dependencies: (i) Each individual in the mutation phase is not generated by standard bit mutation (flipping each bit independently), but by flipping a fixed number $\ell$ of bits. (ii) This $\ell$ is chosen randomly, but all $\lambda$ individuals generated in one mutation phase are generated using the same value of $\ell$.   

\begin{lemma}\label{lem:lower1}
  Let $\lambda$ be an integral function of $n$ with $1 \le \lambda \le n/4$. The probability that the \ga with standard parameters $p=\lambda/n$ and $c=1/\lambda$ has found the optimum within $T = \lfloor n \log(n) / (8\lambda^2) \rfloor$ iterations (equivalent to $2\lambda T$ fitness evaluations) is $\exp(-\Omega(\min\{\lambda T, \sqrt{n}\}))$. In particular, the expected optimization time is $\Omega(n \log(n) / \lambda)$.
\end{lemma}

\begin{proof}
  Using the Chernoff bound of Theorem~\ref{thm:chernoff} \ref{multchernofflower}, we see that with probability $1 - \exp(-\Omega(n))$, the initial search point has at least $n/3$ bits valued zero (``missing bits''). Let us consider what happens in the first $T = \lfloor n \log(n) / (8\lambda^2) \rfloor$ iterations. Denote by $\ell_1, \ldots, \ell_T$ the values of $\ell$ chosen by the algorithm in these iterations. Again, with probability $1 - \exp(-\Omega(n))$, all $\ell_i$ are at most $n/2$ (Chernoff bound of Theorem~\ref{thm:chernoff} \ref{multchernoffupperstrong} and union bound, note that a binomially distributed random variable is a sum of independent $0,1$ random variables). Using these arguments a second time as well as the fact that the $\ell_i$ are independent, we obtain that with probability $1 - \exp(-\Omega(\lambda T))$, we have $\sum_{i = 1}^T \ell_i \le 2\lambda T$. Conditional on none of these exceptional event occurring, the probability that a particular one of the missing bits is never flipped in the mutation phases of the first $T$ iterations is 
  \begin{align*}
  \prod_{i = 1}^T (1 - \ell_i / n)^\lambda 
  & \ge \prod_{i = 1}^T \exp(- 2\ell_i / n)^\lambda 
  = \exp\bigg(- 2 \lambda \sum_{i = 1}^T \ell_i / n\bigg)\\
 & \ge \exp(-4\lambda^2 T / n) 
  \ge n^{-1/2},
  \end{align*}
  where we have used in the first step that $1-c \ge e^{-2c}$ for $0<c<1/2$.
  
  The events that a bit was never flipped in a certain time interval are not independent, since also in a single application of the mutation operator the bits are not treated independently. However, since we always flip a fixed number of bits (not necessarily the same in each iteration), these events are negatively correlated. We omit this proof, because it is technical and lengthy, but neither difficult and nor insightful. From this, we conclude that the probability that there is a missing bit that was never flipped up to iteration $T$ is at least $1 - (1 - n^{-1/2})^{n/3} \ge 1 - \exp(-n^{-1/2})^{n/3} = 1 - \exp(-n^{1/2}/3)$, this time using the estimate $1+x \le e^x$ valid for all $x \in \R$. 
  
  Consequently, with probability at least $1 - \exp(-\Omega(n)) - \exp(-\Omega(\lambda T)) - \exp(-n^{1/2}/3)$, the \ga needs more than $T$ iterations to find the optimum. This immediately implies the claimed bound on the expected optimization time. 
\end{proof}

\subsection{Second Lower Bound}

We prove the second lower bound via drift analysis. We show that the expected fitness increase in each round is $O(\log(\lambda)/\log\log(\lambda))$. Then the additive drift theorem yields the lower bound on the expected optimization time. While our arguments are valid also for constant $\lambda$, in particular the asymptotic notation becomes easier when assuming $\lambda = \omega(1)$. We can make this assumption freely, since for constant $\lambda$ the first lower bound is stronger anyway. 
  
\begin{lemma}\label{lem:lower2}
  Let $\lambda = \omega(1)$. Then the expected optimization time of the \ga with parameters $p=\lambda/n$ and $c=1/\lambda$ is at least $\Omega(n \lambda \log\log(\lambda) / \log(\lambda))$.
\end{lemma}

\begin{proof}
  Let $x$ be any search point different from the optimum. We first analyze what happens in a typical iteration of the main loop of the \ga and later treat the exceptional cases. Let $\ell$, $x'$, and $y^{(1)}, \dots, y^{(\lambda)}$ be as in the pseudo-code of Algorithm~\ref{alg:GA}. 
  
  With probability $1 - \exp(-\Theta(\lambda))$, we have $\ell \le 2 \lambda$ (Chernoff bound of Theorem~\ref{thm:chernoff} \ref{multchernoffupper}). Hence, trivially, $B := \{i \in [n] \mid x_i = 0 \wedge x'_i=1\}$ has cardinality at most $2 \lambda$. Let $Y_j := \{i \in B \mid y^{(j)}_i = 1\}$ the set of new $1$-bits that made it into $y^{(j)}$. We aim at showing that $|Y_j|$ is not very large.  We have $E[|Y_j|] = |B|/\lambda \le 2$. There is nothing to show if $E[|Y_j|]=0$, hence let us assume that $E[|Y_j|] > 0$. Let $r := (2e\ln(\lambda)/\ln\ln(\lambda)) / E[|Y_j|]$ and note that $r \ge e\ln(\lambda)/\ln\ln(\lambda) =: r'$. Then the strong version of the Chernoff bound (Theorem~\ref{thm:chernoff}~\ref{multchernoffupperstrong}) yields    \begin{align*}
  & \Pr\left[|Y_j| \ge 2e \ln(\lambda)/\ln\ln(\lambda) \right] 
  = \Pr\left[|Y_j| \ge  r E[|Y_j|] \right]\\
 & \quad \le (\tfrac{e^{r-1}}{r^r})^{E[|Y_j|]}
 \le (e/r)^{r E[|Y_j|]} = (e/r)^{2r'} \le (e/r')^{2r'}\\
  %hier: e/r=loglog/log und (e/r)^r = \exp(\ln((e/r)))^r
 & \quad = \exp(\ln((e/r')^{2r'})) 
 = \exp(-2e\ln(\lambda)(1-o(1))\\
 & \quad =\lambda^{-2e+o(1)}.
  \end{align*}
  Hence, with probability at least $1 - \lambda^{-2e+1+o(1)}$, none of the $|Y_j|$ is larger than $2r'$, which implies that $\OM(y) \le \OM(x) + 2r'$.
  
  It remains to treat the exceptional cases. If $\ell> 2 \lambda$, which happens with probability at most $\exp(-\Theta(\lambda))$, the expected value of $\ell$ is still only $E[\ell | \ell> 2 \lambda] < 3 \lambda + 1$, and thus $E[\OM(y)] \le \OM(x) + 3\lambda + 1$. The second exceptional case occurs when $\ell \le 2 \lambda$, but (with a probability of at most $\lambda^{-2e+1+o(1)}$) we have $\OM(y) > \OM(x) + 2r'$. In this case, however, we still have $\OM(y) \le \OM(x) + \ell \le \OM(x) + 2\lambda$. 
  
From all this, we see that 
  \begin{align*}
&  E[\OM(y) - \OM(x)]\\
&   \le 
  \left(1 - \exp(-\Theta(\lambda))\right)\\
 & \quad \left((1 - \lambda^{-2	e+1+o(1)})(2e\log(\lambda)/\log\log(\lambda))
  +\lambda^{-2e+1+o(1)}2\lambda \right)\\
  & \quad
   + \exp(-\Theta(\lambda)) (3\lambda+1) \\
 & \le O(\log(\lambda)/\log\log(\lambda)).
  \end{align*}
   We now apply the classic additive drift theorem (Theorem~\ref{thm:drift}) and obtain an expected number of rounds of $\Omega(n \log\log(\lambda) / \log(\lambda))$, equivalent to an optimization time of $\Omega(n \lambda \log\log(\lambda) / \log(\lambda))$.
\end{proof}

\section{Conclusion}

We conducted a tight runtime analysis for the \ga on \onemax giving an improved upper bound for the expected optimization time, the first lower bound (which matches the upper bound for all values of $\lambda$), and a tail bound for the upper tail of the runtime distribution. This analysis both shows that the \ga is faster than what could be shown in~\cite{DoerrDE15}, and it gives the asymptotically optimal value for the off-spring population size $\lambda$, again leading to a super-constant factor speed-up over the runtime stemming from the value giving the best bound in the previous work.

Our sharp bounds also give more insight in the working principles of this algorithm. In particular, we observe that in its crossover phase generating $\lambda$ offspring in parallel often produces at least one offspring that is significantly better that the expected outcome of a crossover application. This allows to gain several (including regularly a super-constant number in the early time of the optimization process) fitness levels in one iteration. This advantage of larger offspring population sizes seems to have been rarely analyzed rigorously (with the analyses of the $(1+\lambda)$ EA in~\cite{JansenJW05,DoerrK15} being the only exceptions known to us). The more common use of larger offspring population sizes in the literature seems to be that an offspring population size of $\lambda$ reduces the waiting time for a fitness level gain by approximately a factor of $\lambda$ (given that this waiting time is large enough). This latter argument, naturally, does not reduce the (total) optimization time (number of fitness evaluations), but only the parallel one (number of generations).

With the proof methods developed in this work, which include a number of clever combinations of drift and Chernoff bounds, we are optimistic that it is now possible to analyze the \ga also on more complicated optimization problems. The first experimental results for several standard test functions~\cite{DoerrDE15} and combinatorial optimization problems~\cite{GoldmanP14} suggest that this is a fruitful direction of research. 

\subsection*{Acknowledgments}
We thank an unknown reviewer for very detailed comments and pointing us to the work Witt~\cite{Witt14}.
}%sloppy

%\bibliographystyle{alpha}
%\bibliography{references}
\newcommand{\etalchar}[1]{$^{#1}$}

\end{document}